\RequirePackage{fix-cm}
\documentclass[smallcondensed]{svjour3}     
\smartqed  
\usepackage{graphicx}

\usepackage{algorithm}
\usepackage{algorithmic}
\usepackage{subfig}

\input{definition}

\newcommand{\domainP}{\Omega}

\newcommand{\domainX}{{\mathcal{X}^n}}
\renewcommand{\AA}{\mathbb{A}}
\newcommand{\domainI}{\AA}

\newcommand{\Hupper}{{\overline{H}}}
\newcommand{\Hlower}{{\underline{H}}}
\newcommand{\uLNML}{{\bar{\Lcal}}}

\begin{document}

\title{
    High-dimensional Penalty Selection via Minimum Description Length Principle
}


\author{Kohei Miyaguchi         \and
        Kenji Yamanishi 
}


\institute{K. Miyaguchi \at
              7-3-1, Bunkyo-ku, Tokyo, Japan, 113-8656\\
              Tel.: +81-3-5841-6009(26899)\\
              \email{kohei\_miyaguchi@mist.i.u-tokyo.ac.jp}           
           \and
           K. Yamanishi \at
              7-3-1, Bunkyo-ku, Tokyo, Japan, 113-8656\\
              Tel.: +81-3-5841-6009(26895)\\
              \email{yamanishi@mist.i.u-tokyo.ac.jp}           
}

\date{Received: date / Accepted: date}

\maketitle

\begin{abstract}
We tackle the problem of penalty selection of regularization on the basis of the minimum description length~(MDL) principle.
In particular, we consider that the design space of the penalty function is high-dimensional.
In this situation, the luckiness-normalized-maximum-likelihood~(LNML)-minimization approach is favorable, because LNML quantifies the goodness of regularized models with any forms of penalty functions in view of the minimum description length principle, and guides us to a good penalty function through the high-dimensional space.
However, the minimization of LNML entails two major challenges:
1) the computation of the normalizing factor of LNML and
2) its minimization in high-dimensional spaces.
In this paper, we present a novel regularization selection method~(MDL-RS),
in which a tight upper bound of LNML~(uLNML) is minimized with local convergence guarantee.
Our main contribution is the derivation of uLNML, which is a uniform-gap upper bound of LNML in an analytic expression.
This solves the above challenges in an approximate manner because it allows us to accurately approximate LNML and then efficiently minimize it.
The experimental results show that MDL-RS improves the generalization performance of regularized estimates specifically when the model has redundant parameters.
\end{abstract}

\section{Introduction}

We are concerned with the problem of learning with redundant models (or hypothesis classes).
This setting is not uncommon in real-world machine learning and data mining problems because the amount of available data is often limited owing to the cost of data collection. In contrast, one can come up with an unbounded number of models that explain the data.
For example, in sparse regression,
one may consider a number of features that are much larger than that in the data, assuming that useful features are actually scarce~\cite{rish2014sparse}.
Another example is statistical conditional-dependency estimation, in which the number of the parameters to estimate is quadratic as compared to the number of random variables, while the number of nonzero coefficients are often expected to be sub-quadratic.

In the context of such a redundant model, there is a danger of overfitting, a situation in which the model fits the present data excessively well but does not generalize well. To address this, we introduce regularization and reduce the complexity of the models by taking the regularized empirical risk minimization (RERM) approach~\cite{shalev2014understanding}.
In RERM, we minimize the sum of the loss and penalty functions to estimate parameters.
However, the choice of the penalty function should be made cautiously
as it controls the bias-variance trade-off of the estimates, and hence has a considerable effect on the generalization capability.

In conventional methods for selecting such hyperparameters,
a two-step approach is usually followed.
First, a candidate set of penalty functions is configured (possibly randomly).
Then, a penalty selection criterion is computed for each candidate and the best one is chosen.
Note that this method can be applied to any penalty selection criteria.
Sophisticated approaches like Bayesian optimization~\cite{mockus2013bayesian} and gradient-based methods~\cite{larsen1996design} also tend to leave the criterion as a black-box.
Although leaving it as a black-box is advantageous in that it works for a wide range of penalty selection criteria, a drawback is that the full information of each specific criterion cannot be utilized. Hence, the computational costs can be unnecessarily large if the design space of the penalty function is high-dimensional.

In this paper, we propose a novel penalty selection method that utilizes information about the objective criterion efficiently on the basis of the minimum description length~(MDL) principle~\cite{rissanen1978modeling}.
We especially focus on the luckiness normalized maximum likelihood~(LNML) code length~\cite{grunwald2007minimum} because the LNML code length measures the complexity of regularized models without making any assumptions on the form of the penalty functions. Moreover, it places a tight bound on the generalization error~\cite{grunwald2017tight}.
However, the actual use of LNML on large models is limited so far.
This is owing to the following two issues.
\begin{itemize}
    \item [I1)] LNML contains a normalizing constant that is hard to compute especially for large models. This tends to make the evaluation of the code length intractable.
    \item [I2)] Since the normalizing term is defined as a non-closed form of the penalty function, efficient optimization of LNML is non-trivial.
\end{itemize}

Next, solutions are described for the above issues.
First, we derive a tight uniform upper bound of the LNML code length, namely uLNML.
The key idea is that, the normalizing constant of LNML, which is not analytic in general, is characterized by the smoothness of loss functions, which can often be upper-bounded by an analytic quantity.
As such, uLNML exploits the smoothness information of the loss and penalty functions to approximate LNML with much smaller computational costs, which solves issue (I1).
Moreover, within the framework of the concave-convex procedure~(CCCP)~\cite{yuille2003concave}, we propose an efficient algorithm for finding a local minimima of uLNML,
i.e., finding a good penalty function in terms of LNML.
This algorithm only adds an extra analytic step to the iteration of the original algorithm for the RERM problem, regardless of the dimensionality of the penalty design.
Thus, issue (I2) is addressed.
We put together these two methods and propose a novel method of penalty selection named \emph{MDL regularization selection~(MDL-RS)}.

We also validate the proposed method from theoretical and empirical perspectives.
Specifically, as our method relies on the approximation of uLNML and the CCCP algorithm on uLNML,
the following questions arise.
\begin{itemize}
    \item [Q1)] How well does uLNML approximate LNML?
    \item [Q2)] Does the CCCP algorithm on uLNML perform well with respect to generalization as
    compared to the other methods for penalty selection?
\end{itemize}
For answering Question (Q1), we show that the gap between uLNML and
LNML is uniformly bounded under smoothness and convexity conditions.
As for Question (Q2), from our experiments on example models involving both synthetic and benchmark datasets, we found that MDL-RS is at least comparable to the other methods and even outperforms them when models are highly redundant as we expected.
Therefore, the answer is affirmative.

The rest of the paper is organized as follows.
In Section~2, we introduce a novel penalty selection criteria, uLNML, with uniform gap guarantees.
Section~3 demonstrates some examples of the calculation of uLNML.
Section~4 provides the minimization algorithm of uLNML and discusses its convergence property.
Conventional methods for penalty selection are reviewed in Section~5.
Experimental results are shown in Section~6.
Finally, Section~7 concludes the paper and discusses the future work.

\section{Method: Analytic Upper Bound of LNMLs}

In this section,
we first briefly review the definition of RERM and the notion of penalty selection.
Then, we introduce the LNML code length.
Finally, as our main result,
we show an upper bound of LNML, uLNML, and the associated minimization algorithm.
Theoretical properties and examples of uLNML are presented in the last part.

\subsection{Preliminary: Regularized Empirical Risk Minimization (RERM)}

Let $f_{X}:\RR^p\to \RRbar (=\RR\cup\cbr{\infty})$ be an extended-value loss function of parameter $\theta\in\RR^p$ with respect to data $X=(x_1,\ldots,x_n)\in\domainX$.
We assume $f_X(\theta)$ is a log-loss (but not limited to i.i.d. loss), i.e.,
it is normalized with respect to some base measure $\nu$ over $\domainX$, where $\int_{\domainX} \exp\cbr{-f_X(\theta)}d\nu(X)= 1$ for all $\theta$ in some closed subset $\domainP_0\subset\RR^p$.
Here, $x_i$ can be a pair of a datum and label $(x_i, y_i)$ in the case of supervised learning.
We drop the subscript $X$ and just write $f(\theta)=f_X(\theta)$ if there is no confusion.
The regularized empirical risk minimization~(RERM) with domain $\domainP\subset\domainP_0$ is defined as the minimization of
the sum of the loss function and a penalty function $g:\RR^p\times \domainI\to \RRbar$ over $\domainP$,
\begin{align}
    \mathrm{RERM(\lambda)}:\qquad \mathop{\mathrm{minimize}}\  f_X(\theta) + g(\theta, \lambda)\quad \mathrm{s.t.}\quad \theta\in \domainP,
    \label{eq:RERM}
\end{align}
where $\lambda\in \domainI\subset\RR^d$ is the only hyperparameter that parametrizes the shape of penalty on $\theta$. Let $F_X(\lambda)=\min_{\theta\in\domainP} f_X(\theta) + g(\theta,\lambda)$ be the minimum value of the RERM.
We assume that the minimizer always exists, and denote one of them as $\thetahat(X,\lambda)$.
Here, we focus on a special case of RERM in which the penalty is linear to $\lambda$,
\begin{align}
    g(\theta,\lambda) = \sum_{j=1}^d \lambda_j g_j(\theta),\quad \lambda_j\ge 0\quad(j=1,\ldots, d),
    \label{eq:linear_penalty}
\end{align}
and $\domainI\subset\RR^d_+$ is a convex set of positive vectors.
Let $\lambda_\star=\inf \domainI$ be the infimum of $\domainI$.
We also assume that the following regularity condition holds:

\begin{assumption}[Regular penalty functions]\label{asm:regular_penalty}
    If $\thetahat(X, \lambda) \in \interior{\domainP}$, then,
    for all $\lambda'\ge \lambda$, $\thetahat(X, \lambda') \in \interior{\domainP}$.
\end{assumption}

 Regularization is beneficial from two perspectives.
It improves the condition number of the optimization problem, and hence enhances the numerical stability of the estimates.
It also prevents the estimate from overfitting to the training data $X$, and hence reduces generalization error.

 However, these benefits come with an appropriate penalization.
If the penalty is too large, the estimate will be biased.
If the penalty is too small, the regularization no longer takes effect and
the estimate is likely to overfit.
Therefore, we are motivated to select good $\lambda$ as a function of data $X$.

\subsection{Luckiness Normalized Maximum Likelihood (LNML)}

In order to select an appropriate hyperparameter $\lambda$,
we introduce the luckiness normalized maximum likelihood (LNML) code length as a criterion for the penalty selection.
The LNML code length associated with $\mathrm{RERM(\lambda)}$
is given by
\begin{align}
    \Lcal(X\mid \lambda) \eqdef \min_{\theta\in\domainP} f_X(\theta) + g(\theta, \lambda) + \log Z(\lambda),
    \label{eq:LNML}
\end{align}
where $Z(\lambda)\eqdef \int \max_{\theta\in\domainP} \exp\cbr{-f_X(\theta) - g(\theta, \lambda)} d\nu(X)$ is
the normalizing factor of LNML.

The normalizing factor $Z(\lambda)$ can be seen as a penalization of the complexity of $\mathrm{RERM(\lambda)}$.
It quantifies how much $\mathrm{RERM(\lambda)}$ will overfit to random data.
If the penalty $g$ is small such that the minimum in \eqref{eq:RERM} always takes a low value for all $X\in\domainX$, $Z(\lambda)$ becomes large. 
Specifically, any constant shift on the RERM objective,
which does not change the RERM estimator $\thetahat$, does not change LNML since $Z(\lambda)$ cancels it out.
Note that LNML is originaly derived by generalization of the Shtarkov's minimax coding strategy~\cite{shtar1987universal}, \cite{grunwald2007minimum}.
Moreover, recent advances in the analysis of LNML show that it bounds the generalization error of $\thetahat(X,\lambda)$~\cite{grunwald2017tight}.
Thus, our primary goal is to minimize the LNML code length~\eqref{eq:LNML}.

\subsection{Upper Bound of LNML (uLNML)}

The direct computation of the normalizing factor $Z(\lambda)$ requires integration of the RERM objective~\eqref{eq:RERM} over all possible data, and hence, direct minimization is often intractable.
To avoid computational difficulty, we introduce an upper bound of $Z(\lambda)$ that is analytic with respect to $\lambda$. Then, adding the upper bound to the RERM objective, we have an upper bound of the LNML code length itself.

To derive the bound, let us define $\Hupper$-upper smoothness condition of the loss function $f(\theta)$.

\begin{definition}[$\Hupper$-upper smoothness]
    A function $f:\RR^p\to \RRbar$ is $\Hupper$-upper smooth, or $(\Hupper, c, r)$-upper smooth to avoid any ambiguity, over $\domainP$ for some $\Hupper\in\SS_{+}^{p}$,
    if there exists a constant $c\ge 0$, vector-valued function $\xi:\RR^p\to\RR^p$, and monotone increasing function $r:\RR\to\RRbar_+$ such that
    \begin{align*}
        f(\psi)-f(\theta)\le c + \inner{\xi(\theta)}{\psi-\theta} + \frac{1}{2}\norm{\psi-\theta}_{\Hupper}^2 + r(\norm{\theta-\psi}^2),\quad \forall \psi\in \RR^p, \forall\theta\in\domainP,
    \end{align*}
    where 
    $\norm{\psi-\theta}_{\Hupper}^2\eqdef (\psi-\theta)^\top \Hupper (\psi-\theta)$ and $r(t^2)=o(t^2)$.
\end{definition}

Note that the $\Hupper$-upper smoothness is a condition that is weaker than that of standard smoothness. In particular, $\rho$-smoothness implies $(\rho I_p, 0, 0)$-upper smoothness.
Moreover, it is noteworthy that all the bounded functions are upper smooth with respective $(\Hupper, c, r)$.

Now, we show the main theorem that bounds $Z(\lambda)$.
The theorem states that the upper bound depends on $f_X$ and $g$ only through their smoothness.

\begin{theorem}[Upper bound of $Z(\lambda)$]\label{thm:uLNML}
    Suppose that $f_X(\cdot)$ is $(\Hupper_0, c_0, r)$-upper smooth with respect to $\theta\in\domainP$ for all $X\in\domainX$,
    and that $g_j(\cdot)$ is $(\Hupper_j, c_j, 0)$-upper smooth for $j=1,\ldots,d$.
    Then, for every symmetric neighbor of the origin $U\subset\RR^p$ where $\domainP+U\subset \Omega_0$, we have
    \begin{align}
        Z(\lambda) \le \frac{e^{c(\lambda)}\det \rbr{\Hupper(\lambda)}^{\frac12}}{\sqrt{2\pi}^pR(\Hupper_0;U)}  \int_{\domainP + U} e^{-g(\theta, \lambda)} d\theta\eqdef \Zbar(\lambda),
    \end{align}
    where $\Hupper(\lambda)\eqdef \Hupper_0+\sum_{j=1}^d\lambda_j \Hupper_j$,
    $c(\lambda)\eqdef c_0+\sum_{j=1}^d\lambda_j c_j$ and $R(H;U)\eqdef \EE_{z\sim\Ncal_p[\zero, H^{-1}]} \sbr{\one_U(z) e^{-r\rbr{\norm{z}^2}} }$.
\end{theorem}
\begin{proof}
    Let $q_\lambda(X)\eqdef \int_{\domainP+U} \exp\cbr{-f_X(\theta)-g(\theta, \lambda)}d\theta$. First, by H\"older's inequality, we have
{
    \newcommand{\ratio}{A}
    \newcommand{\integral}{B}
    \begin{align*}
        Z(\lambda)
        &= \int_\domainX \max_{\theta\in\domainP}\exp\cbr{-f_X(\theta)-g(\theta, \lambda)} d\nu(X)
      \\&\le \sup_{X\in\domainX}\max_{\theta\in\domainP} \underbrace{\frac{\exp\cbr{-f_X(\theta)-g(\theta, \lambda)}}{q_\lambda(X)}}_{\ratio} \underbrace{\int_\domainX q_\lambda(X)d\nu(X)}_{\integral}.
    \end{align*}
    Then, we will bound $\ratio$ and $\integral$ in the right-hand side, respectively.
    Since we assume that $f_X(\theta)$ is a logarithmic loss if $\theta\in \Omega_0$, the second factor is simply evaluated using Fubini's theorem,
    \begin{align*}
        \integral
        &=\iint_{(\domainP+U)\times \domainX} \exp\cbr{-f_X(\theta)-g(\theta, \lambda)}d\theta d\nu(X)
      \\&= \int_{\domainP+U} e^{-g(\theta, \lambda)}d\theta.
    \end{align*}
    On the other hand, by $\Hupper(\lambda)$-upper smoothness of $f(\theta) + g(\theta, \lambda)$, we have
    \begin{align*}
        A^{-1}
        &= q_\lambda(X)\exp\cbr{f_X(\theta)+g(\theta)}
      \\&=\int_{\Omega+U} \exp\cbr{f_X(\theta) + g(\theta, \lambda)- f_X(\psi)  - g(\psi, \lambda)} d\psi
      \\&\ge \int_{\Omega+U} \exp\cbr{-c(\lambda) - \inner{\xi(\theta)}{\psi-\theta} - \frac12 \norm{\psi - \theta}_{\Hupper(\lambda)}^2-r(\norm{\psi-\theta}^2)} d\psi
      \\&\ge e^{-c(\lambda)}\int_{U} \exp\cbr{-\inner{\xi(\theta)}{z}- \frac12 \norm{z}_{\Hupper(\lambda)}^2-r(\norm{z}^2)} dz
      \\&\ge e^{-c(\lambda)}\int_{U} \exp\cbr{- \frac12 \norm{z}_{\Hupper(\lambda)}^2-r(\norm{z}^2)} dz
      \\&= e^{-c(\lambda)}\frac{\sqrt{2\pi}^p}{\det \Hupper(\lambda) ^{\frac12}} R(\Hupper(\lambda);U)
      \\&\ge  e^{-c(\lambda)}\frac{\sqrt{2\pi}^p}{\det \Hupper(\lambda) ^{\frac12}} R(\Hupper_0;U).
    \end{align*}
    This concludes the proof.
}
\end{proof}

The upper bound in Theorem~\ref{thm:uLNML} can be easily computed by ignoring the constant factor $R(\Hupper_0, U)^{-1}$ given the upper smoothness of $f_X$ and $g(\cdot, \lambda)$.
In particular, the integral $\int_{\domainP+U} e^{-g(\theta, \lambda)}d\theta$ can be evaluated in a closed form if one chooses a suitable class of penalty functions with a suitable neighbor $U$; for e.g., linear combination of quadratic functions with $U=\RR^p$.
Therefore, we adopt this upper bound (except with constant terms) as an alternative of the LNML code length, namely \emph{uLNML},
\begin{align}
    \uLNML(X|\lambda) &\eqdef \min_{\theta\in\domainP} f_X(\theta)+g(\theta,\lambda) + \Zbar(\lambda)
    \label{eq:uLNML}
                    \\&=\min_{\theta\in\domainP} f_X(\theta)+g(\theta,\lambda) +
                    \\& \qquad c(\lambda) + \frac12 \log \det \Hupper(\lambda) + \log \int_{\domainP+U} e^{-g(\psi,\lambda)}d\psi+\mathrm{const.},
    \nonumber
\end{align}
where the symmetric set $U$ is fixed beforehand.
In practice, we recommend just taking $U=\RR^p$ because uLNML with $U=\RR^p$ bounds uLNMLs with $U\neq\RR^p$.
However, for the sake of the later analysis, we leave $U$ to be arbitrary.

We present two useful specializations of uLNML with respect to the penalty function $g(\theta, \lambda)$.
One is the Tikhonov regularization, known as the $\ell2$-regularization.

\begin{corollary}[Bound for Tikhonov regularization]\label{cor:l2}
    Suppose that $f_X(\theta)$ is $(\Hupper_0, c_0, r)$-upper smooth for all $X\in\domainX$ and $g(\theta, \lambda)=\frac12 \sum_{j=1}^p \lambda_j \theta_j^2$ where $\lambda_j> 0$ for all $1\le j\le p$. Then, we have
    \begin{align*}
        Z(\lambda) \le \frac{e^{c_0}}{R(\Hupper_0;\RR^p)} \sqrt{\frac{\det (\Hupper_0+\diag \lambda)}{\det \diag \lambda }}.
    \end{align*}
\end{corollary}
\begin{proof}
    The claim follows from setting $U=\RR^p$ in Theorem~\ref{thm:uLNML} and the fact that $g(\cdot, \lambda)$ is $(\diag \lambda, 0, 0)$-upper smooth.
\end{proof}

The other one is that of lasso~\cite{tibshirani1996regression}, known as $\ell1$-regularization.
It is useful if one needs sparse estimates $\thetahat(X,\lambda)$.

\begin{corollary}[Bound for lasso]\label{cor:lasso}
    Suppose that $f_X(\theta)$ is $(\Hupper, c, r)$-upper smooth for all $X\in\domainX$ and that $g(\theta, \lambda)=\sum_{j=1}^p \lambda_j \abs{\theta_j}$, where $\lambda > 0$ for all $1\le j\le p$. Then, we have
    \begin{align*}
        Z(\lambda) \le \frac{e^c}{R(\Hupper;\RR^p)}\sqrt{\frac{e}{2\pi}}^p \sqrt{\frac{\det (\Hupper+(\diag \lambda)^2)}{\det (\diag \lambda)^2 }}.
    \end{align*}
\end{corollary}
\begin{proof}
    As in the proof of Corollary~\ref{cor:l2}, it follows from Theorem~\ref{thm:uLNML} and the fact that $g(\cdot, \lambda)$ is $((\diag \lambda)^2, 1/2, 0)$-upper smooth.
\end{proof}

Finally, we present a useful extension for RERMs with Tikhonov regularization, which contains the inverse temperature parameter $\beta\in[a, b]\ (0<a\le b)$ as a part of the parameter:
\begin{align}
    &f_X(\beta, \theta)=\beta \ftil_X(\theta) + \log C(\beta),
    \label{eq:temperature_RERM_f}
    \\&g(\beta, \theta, \lambda)=\beta \gtil(\theta, \lambda)=\beta \sum_{j=1}^d \frac{\lambda_j}2 \theta_j^2,
    \label{eq:temperature_RERM_g}
\end{align}
where $C(\beta)\eqdef \int e^{-\beta \ftil_X(\theta)}d\nu(X)<\infty$ is the normalizing constant of the loss function.
Here, we assume that $C(\beta)$ is independent of the non-temperature parameter $\theta$.
Interestingly, the normalizing factor of uLNML for a variable temperature model~\eqref{eq:temperature_RERM_f}, \eqref{eq:temperature_RERM_g} is bounded with the same bound as that for the fixed temperature models in Corollary~\ref{cor:l2} except for a constant.

\begin{corollary}[Bound for variable temperature model]\label{cor:bound_for_temperature}
    Let $(\beta, \theta)\in[a, b]\times \domainP$ be a parameter of the model~\eqref{eq:temperature_RERM_f}.
    Suppose that $\ftil_X(\theta)$ is $(\Hupper_0, c_0, r)$-upper smooth for all $X\in\domainX$.
    Then, there exists a constant $C_{[a, b]}$ such that
    LNML's normalizing factor for the RERM \eqref{eq:temperature_RERM_f}, \eqref{eq:temperature_RERM_g} is bounded as
    \begin{align*}
        Z(\lambda)\le \frac{C_{[a, b]}e^{(b+a/2)c_0}}{R(\frac{a}2\Hupper_0;\RR^p)}
        \sqrt{\frac{\det \rbr{\Hupper_0 + \diag \lambda}}{\det \diag \lambda}}.
    \end{align*}
\end{corollary}
\begin{proof}
    Let $\Ftil_X(\lambda)=\min_{\theta\in\domainP} \ftil_X(\theta)+\gtil(\theta, \lambda)$ and 
    $\qtil_\lambda(X)=\int_{a/2}^{b+a/2} \exp\cbr{-\beta\Ftil_X(\lambda)-\log C(\beta)}d\beta$.
    Note that $C(\beta)$ is a continuous function, and hence bounded over $W=[a/2, b+a/2]$,
    which implies that it is upper smooth.
    Let $(h_\beta, c_\beta, r_\beta)$ be the upper smoothness of $\log C(\beta)$ over $W$.
    Then,
    \begin{align*}
        Z(\lambda)
        &=\int \max_{\beta\in[a, b],\ \theta\in\domainP} \exp\cbr{-\beta\sbr{\ftil_X(\theta)-\gtil(\theta, \lambda)}-\log C(\beta)} d\nu(X)
        \\&\le \max_{\beta\in[a, b]} \sup_{X\in\domainX}
        \exp\cbr{-\beta\Ftil_X(\lambda)-\log C(\beta)-\log \qtil_\lambda(X)}
        \int \qtil_\lambda(X) d\nu(X)
        \\&\le \frac{e^{c_\beta}}{R_\beta(h_\beta;W_t)} \sqrt{\frac{h_\beta}{2\pi}}
        \int_W d\beta \int \max_{\theta\in\domainP}\exp\cbr{
        -\beta\ftil_X(\theta)-\beta\gtil(\theta, \lambda)-\log C(\beta)
    } d\nu(X)
        \\&\le \frac{e^{c_\beta}}{R_\beta(h_\beta;W_t)} \sqrt{\frac{h_\beta}{2\pi}}
        \int_{W}
        \frac{e^{\beta c_0}}{R(\beta \Hupper_0;\RR^p)}\sqrt{\frac{\det \beta\rbr{\Hupper_0 + \diag \lambda}}{\det \beta \diag \lambda}}d\beta
        \\&= 
        \frac{C_{[a, b]}e^{(b+a/2) c_0}}{R(\frac{a}{2} \Hupper_0;\RR^p)}
        \sqrt{\frac{\det \rbr{\Hupper_0 + \diag \lambda}}{\det \diag \lambda}}.
        \end{align*}
\end{proof}

\subsection{Gap between LNML and uLNML}

In this section, we evaluate the tightness of uLNML.
To this end, we now bound LNML from below.
The lower bound is characterized with strong convexity of $f_X$ and $g(\cdot, \lambda)$.

\begin{definition}[$\Hlower$-strong convexity]
    A function $f(\theta)$ is $\Hlower$-strong-convex if there exists a constant $c\ge 0$ and a vector-valued function $\xi:\RR^p\to\RR^p$ such that
    \begin{align*}
        f(\psi)-f(\theta)\ge \inner{\xi(\theta)}{\psi-\theta} + \frac{1}{2}\norm{\psi-\theta}_\Hlower^2, \quad \forall \psi\in\RR^p, \forall\theta\in\domainP.
    \end{align*}
\end{definition}

Note that $\Hlower$-strong convexity can be seen as the matrix-valued version of the standard strong convexity.
Now, we have the following lower bound of $Z(\lambda)$.

\begin{theorem}[Lower bound of $Z(\lambda)$]\label{thm:lower_bound_LNML}
    Suppose that $f_X$ is $\Hlower_0$-strongly convex and 
    $g_j$ is $\Hlower_j$-strongly convex,
    where $\Hlower_0\in\SS_{++}^p$ and $\Hlower_j\in\SS_{+}^p$ for all $j=1,\ldots,d$.
    Then, for every set of parameters $V\subset\RR^p$, we have
    \begin{align}
        Z(\lambda)\ge \frac{\det \Hlower(\lambda)}{\sqrt{2\pi}^p} \inf_{\psi\in V} T(\psi) \int_{V} e^{-g(\theta,\lambda)}d\theta,
        \label{eq:lower_bound_LNML}
    \end{align}
    where $\Hlower(\lambda)\eqdef \Hlower_0+\sum_{j=1}^d \lambda_j \Hlower_j$ and $T(\psi)\eqdef\int_{X:\thetahat(X,\lambda_\star)\in\interior{\domainP}} e^{-f_X(\psi)}d\nu(X)$.
\end{theorem}
\begin{proof}
    Let $D(\lambda)\eqdef \myset{X\in\domainX}{\thetahat(X,\lambda)\in\interior{\domainP}}$.
    Let $q_\lambda(X)\eqdef \int_{V} \exp\cbr{-f_X(\theta)-g(\theta, \lambda)}d\theta$. First, from the positivity of $q_\lambda$, we have
{
    \newcommand{\ratio}{A}
    \newcommand{\integral}{B}
    \begin{align*}
        Z(\lambda)
        &= \int_\domainX \max_{\theta\in\domainP}\exp\cbr{-f_X(\theta)-g(\theta, \lambda)} d\nu(X)
      \\&\ge \int_{D(\lambda)} \max_{\theta\in\interior{\domainP}}\exp\cbr{-f_X(\theta)-g(\theta, \lambda)} d\nu(X)
      \\&\ge \inf_{X\in D(\lambda)}\underbrace{\max_{\theta\in\interior{\domainP}} \frac{\exp\cbr{-f_X(\theta)-g(\theta, \lambda)}}{q_\lambda(X)}}_{\ratio} \underbrace{\int_{D(\lambda)} q_\lambda(X)d\nu(X)}_{\integral}.
    \end{align*}
    Then, we bound from below $\ratio$ and $\integral$ in the right-hand side, respectively.
    Since we assumed that $f_X(\theta)$ is a logarithmic loss, the second factor is simply evaluated using Fubini's theorem,
    \begin{align*}
        \integral
        &=\iint_{V \times D(\lambda)} \exp\cbr{-f_X(\theta)-g(\theta, \lambda)}d\theta d\nu(X)
      \\&= \int_V T(\theta) e^{-g(\theta, \lambda)}d\theta.
      \\&\ge \inf_{\psi\in V} T(\psi) \int_V  e^{-g(\theta, \lambda)}d\theta,
    \end{align*}
    where the first inequality follows from Assumption~\ref{asm:regular_penalty}.
    On the other hand, by the $\Hlower(\lambda)$-strong convexity of $f(\theta) + g(\theta, \lambda)$, we have
    \begin{align*}
        A^{-1}
        &= \min_{\theta\in\interior{\domainP}}q_\lambda(X)\exp\cbr{f_X(\theta)+g(\theta)}
      \\&=\int_{V} \exp\cbr{\min_{\theta\in\interior{\domainP}}f_X(\theta) + g(\theta, \lambda)- f_X(\psi)  - g(\psi, \lambda)} d\psi
      \\&\le \int_{V} \exp\cbr{-\inner{\xi(\thetahat(X,\lambda))}{z}- \frac12 \norm{z}_{\Hlower(\lambda)}^2} dz
      \\&= \frac{\sqrt{2\pi}^p}{\det \Hlower(\lambda) ^{\frac12}},
    \end{align*}
    for all $X\in D(\lambda)$.
    Here, we exploit the fact that we can take $\xi(\thetahat(X,\lambda))=0$, if $X\in D(\lambda)$.
    This concludes the proof.
}
\end{proof}

Combining the result of Theorem~\ref{thm:lower_bound_LNML} with Theorem~\ref{thm:uLNML}, we have a uniform gap bound of uLNML for quadratic penalty functions.

\begin{theorem}[Uniform gap bound of uLNML]\label{thm:gap}
    Suppose that the assumptions made in Theorem~\ref{thm:uLNML} and Theorem~\ref{thm:lower_bound_LNML} are satisfied.
    Suppose that the penalty function is quadratic, i.e.,
    $\Hupper_j=\Hlower_j$ and $c_j=0$ for all $j=1,\ldots,d$.
    Then, the gap between LNML and uLNML is uniformly bounded as, for all $X\in\domainX$ and $\lambda\in\domainI$,
    \begin{align}
        \uLNML(X|\lambda)-\Lcal(X|\lambda)
        \le c_0 + \frac12 \log \frac{\det\Hupper_0}{\det \Hlower_0}
        - \log R(\Hupper_0;U)
        -\log \inf_{\psi\in\domainP+U} T(\psi),
        \label{eq:gap_bound}
    \end{align}
    where $R(\Hupper_0;U)$ and $T(\psi)$ are defined as in the preceding theorems.
\end{theorem}
\begin{proof}
    From Theorem~\ref{thm:uLNML} and Theorem~\ref{thm:lower_bound_LNML}, we have
    \begin{align*}
        \uLNML(X|\lambda)-\Lcal(X|\lambda)
        &\le \log \Zbar(\lambda) - \log Z(\lambda)
        \\&\le c(\lambda) + \frac12 \log \frac{\det \Hupper(\lambda)}{\det \Hlower(\lambda)}
        -\log R(\Hupper_0;U)- \log \inf_{\psi\in V}T(\psi)
        \\&\qquad +\log \frac{\int_{\domainP+U}e^{-g(\theta, \lambda)}d\theta}{\int_{V}e^{-g(\theta, \lambda)}d\theta},
    \end{align*}
    where $c(\lambda)=c_0$ from the assumption.
    Taking $V=\domainP+U$ to cancel out the last term,
    we have
    \begin{align}
        \uLNML(X|\lambda)-\Lcal(X|\lambda)
        &\le c_0 + \frac12 \log \frac{\det \Hupper(\lambda)}{\det \Hlower(\lambda)}
        - \log R(\Hupper_0;U) - \inf_{\psi\in \domainP+U} \log T(\psi).
        \label{eq:gap_bound_tmp1}
    \end{align}
    Let $\kappa(Q) \eqdef \log \frac{\det(\Hupper_0 + Q)}{\det(\Hlower_0 + Q)}$ for $Q\in\SS_{+}^p$ and
    let 
    $\underline{Q}=\Hupper_0^{-\frac12}Q\Hupper_0^{-\frac12}$
    and $\overline{Q}=\Hlower_0^{-\frac12}Q\Hlower_0^{-\frac12}$.
    Then, we have
    \begin{align*}
        \pwrt{t}\kappa(tQ)
        &=\tr\rbr{(\Hupper_0+tQ)^{-1}Q - (\Hupper_0+tQ)^{-1}Q}
        \\&= \tr\rbr{(I+t\underline{Q})^{-1}\underline{Q} - (I+t\overline{Q})^{-1}\overline{Q}}
        \le 0,
    \end{align*}
    where the last inequality follows from $\underline{Q}\preceq \overline{Q}$.
    This implies that
    \begin{align*}
        \log \frac{\det \Hupper(\lambda)}{\det \Hlower(\lambda)}
        = \kappa\rbr{\sum_{j=1}^d \Hupper_j}
        \le \kappa(O)
        = \log \frac{\det \Hupper_0}{\det \Hlower_0},
    \end{align*}
    which, combined with~\eqref{eq:gap_bound_tmp1}, completes the proof.
\end{proof}

The theorem implies that uLNML is a tight upper bound of the LNML code length if $f_X$ is strongly convex.
Moreover, the gap bound~\eqref{eq:gap_bound} can be utilized for choosing a good neighbor $U$.
Suppose that there is no effective boundary in the parameter space, $\domainP=\interior{\domainP}$.
Then, we can simplify the gap bound and the optimal neighbor $U$ is explicitly given.

\begin{corollary}[Uniform gap bound for no-boundary case]\label{cor:uniform_gap}
    Suppose that the assumptions made in Theorem~\ref{thm:gap} is satisfied.
    Then, if $\domainP=\interior{\domainP}$, we have a uniform gap bound
    \begin{align}
        \uLNML(X|\lambda)-\Lcal(X|\lambda)
        \le c_0 + \frac12 \log \frac{\det \Hupper_0}{\det \Hlower_0}
        - \log R(\Hupper_0;U)
        \label{eq:gap_bound_no_boundary}
    \end{align}
    for all $X\in\domainX$ and all $\lambda\in\domainI$.
    This bound is minimized with maximum $U$, i.e., 
    $U=\bigcap_{\theta\in\domainP}\domainP_0-\cbr{\theta}$.
\end{corollary}
\begin{proof}
    According to Theorem~\ref{thm:gap}, it suffices to show that $T(\psi;\lambda)\equiv1$.
    From the existence of the RERM estimate in $\domainP$,
    we have $\thetahat(X, \lambda)\in \domainP=\interior{\domainP}$ for all $X\in\domainX$ and all $\lambda\in\domainI$.
    Therefore, $D(\lambda)=\domainX$, and hence, $T(\psi;\lambda)=\int_{D(\lambda)}e^{-f_X(\psi)}d\nu(X)\equiv1$,
    where $D(\lambda)$ is defined in the proof of Theorem~\ref{thm:lower_bound_LNML}.
    The second argument follows from the monotonicity of $R(H;\cdot)$.
\end{proof}

As a remark, if we assume in addition that $f_X$ is a smooth \iid loss,
i.e., $f_X=\sum_{i=1}^n f_{x_i}$ and $c_0=0$,
the gap bound is also uniformly bounded with respect to the sample size $n$.
This is derived from the fact that the right-hand side of $\eqref{eq:gap_bound_no_boundary}$ turns out to be
\begin{align*}
    \log \frac{\det n\Hupper}{\det n\Hlower}-\log \EE_{z\sim \Ncal_m[\zero, \frac1n\Hupper_0^{-1}]} \sbr{\one_U(z)e^{-r(\|z\|^2)}}\overset{n\to \infty}{\longrightarrow} \log \frac{\det \Hupper}{\det \Hlower},
\end{align*}
which is constant independent of $n$.

\subsection{Discussion}

In previous sections, we derived an upper bound of the normalizing constant $Z(\lambda)$ and defined an easy-to-compute alternative for the LNML code length, called uLNML. We also stated uniform gap bounds of uLNML for smooth penalty functions.
Note that uLNML characterizes $Z(\lambda)$ with upper smoothness of the loss and penalty functions.
This is both advantageous and disadvantageous.
The upper smoothness can often be easily computed even for complex models like deep neural networks.
This makes uLNML applicable to a wide range of loss functions.
On the other hand, if the Hessian of the loss function drastically varies across $\domainP$,
the gap can be considerably large.
In this case, one can tighten the gap by reparametrizing $\domainP$ to make the Hessian as uniform as possible.

The derivation of uLNML relies on the upper smoothness of the loss and penalty functions. In particular, our current analysis on the uniform gap guarantee given by Theorem~\ref{thm:gap} holds if the penalty function is smooth, i.e., $c_j=0$.
This is violated if one employs the $\ell1$-penalties.

It should be noted that there exists approximation of LNML originally given by \cite{rissanen1996fisher} for a special case and then generalized by \cite{grunwald2007minimum}.
This approximates LNML except for the $o(1)$ term with respect to $n$,
\begin{align*}
    \Lcal(X|\lambda) &= \min_{\theta\in\domainP} f_X(\theta) + g(\theta,\lambda) + \frac{p}{2}\log \frac{n}{2\pi} +\log \int_\domainP \sqrt{\det I(\psi)} e^{-g(\psi, \lambda)}d\psi+o(1),
\end{align*}
where $I(\psi)\eqdef \int \sbr{\nabla f_X(\theta) \nabla f_X(\theta)^\top } e^{-f_X(\theta)} d\nu(X)$ denotes the Fisher information matrix.
A notable difference between this approximation and uLNML is in the boundedness of their approximation errors.
The above $o(1)$ term is not necessarily uniformly bounded with respect to $\lambda$, and actually it diverges for every fixed $n$ as $\norm{\lambda}\to \infty$ in the case of, for example, the Tikhonov regularization.
This is in contrast to uLNML in that
the approximation gap of uLNML is uniformly bounded with respect to
$\lambda$ according to Corollary~\ref{thm:gap},
and it does not necessarily go to zero as $n\to\infty$.
This difference can be significant, especially in the scenario of penalty selection,
where one compares different $\lambda$ while $n$ is fixed.


\section{Examples of uLNML}\label{sec:example}

In the previous section, we have shown that the normalizing factor of LNML is bounded if
the upper smoothness of $f_X(\theta)$ is bounded.
The upper smoothness can be easily characterized for a wide range of loss functions.
Since we cannot cover all of it here, we present below a few examples that will be used in the experiments.

\subsection{Linear Regression}

Let $X\in\RR^{n\times m}$ be a fixed design matrix and $y\in\RR^n=\domainX$ represent the corresponding target variables.
Then, we want to find $\beta\in\RR^m$ such that $y\approx X\beta$.
We assume that the `useful' predictors may be sparse,
and hence, most of the coefficients of the best $\beta$ for generalization may be close to zero.
As such, we are motivated to solve the ridge regression problem:
\begin{align}
    \min_{\sigma^2\in [a, b],\;\beta\in\RR^d} -\log p(y|X,\beta,\sigma^2) + \frac1{2\sigma^2} \sum_{j=1}^p \lambda_j \beta_j^2,
    \label{eq:linear_regression}
\end{align}
where $-\log p(y|X,\beta,\sigma^2)=\frac1{2\sigma^2} \norm{y-X\beta}^2 + \frac{n}{2}\log 2\pi\sigma^2$.
According to Corollary~\ref{cor:bound_for_temperature}, the uLNML of the ridge regression is given by
\begin{align*}
    \uLNML(X|\lambda) &= \min_{\sigma^2\in [a, b],\;\beta\in\RR^d}-\log p(y|X,\beta,\sigma^2) + 
    \frac1{2\sigma^2} \sum_{j=1}^p \lambda_j \beta_j^2 +
    \\&\qquad \frac12 \log \frac{\det (C + \diag \lambda)}{\det \diag \lambda}+\mathrm{const.},
\end{align*}
where $C\eqdef X^\top X$.
Note that the above uLNML is uniformly bounded because the normalizing constant of the LNML code length of \eqref{eq:linear_regression} is bounded from below with a fixed variance $\sigma^2\in[a, b]$ that exactly evaluates to $\frac12 \log \frac{\det (C + \diag \lambda)}{\det \diag \lambda}$.

\subsection{Conditional Dependence Estimation}

Let $X=\rbr{x_1, x_2,\ldots, x_n}^\top\in \RR^{n\times m}=\domainX$ be a sequence of $n$ observations independently drawn from the $m$-dimensional Gaussian distribution
$\Ncal_m[0, \Sigma]$.
We assume that the conditional dependence among the $m$ variables in $X$ is scarce,
which means that most of the coefficients of precision $\Theta=\Sigma^{-1}\in\RR^{m\times m}$ are (close to) zero.
Thus, to estimate the precision matrix $\Theta$,
we penalize the nonzero coefficients and consider the following RERM
\begin{align}
    \min_{\Theta\in \domainP} -\log p(X|\Theta) + \sum_{i\neq j} \lambda_{ij} \Theta_{ij}^2,
    \label{eq:graphical_l2}
\end{align}
where $\domainP=\myset{\Theta\in\SS_{++}^m}{\diag \Theta^{-1}\le R\one_m}$ and $-\log p(X|\Theta)= \frac{1}{2}\cbr{\tr X^\top X \Theta - n\log \det 2\pi\Theta}$ denotes the probability density function of the Gaussian distribution.
As it is an instance of the Tikhonov regularization,
from Corollary~\ref{cor:l2} with $\Hupper_0=mnR^2I_{m^2}$,
the uLNML for the graphical model is given by
\begin{align*}
    \uLNML(X|\lambda) = \min_{\Theta\in \domainP} -\log p(X|\Theta) + \sum_{i\neq j} \lambda_{ij} \Theta_{ij}^2
    + \frac12 \log \frac{mnR^2+\lambda_{ij}}{\lambda_{ij}}.
\end{align*}

\section{Minimization of uLNML}

Given data $X\in\domainX$, we want to minimize uLNML~\eqref{eq:uLNML} with respect to $\lambda\in\domainI$ as it bounds the LNML code length, which is a measure of the goodness of the penalty with respect to the MDL principle~\cite{rissanen1978modeling}, \cite{grunwald2007minimum}. Furthermore, it bounds the risk of the RERM estimate $\EE_Y f_Y(\thetahat(X, \lambda))$~\cite{grunwald2017tight}.
The problem is that grid-search-like algorithms are inefficient since the dimensionality of the domain $\domainI\subset \RR^d$ is high.

In order to solve this problem,
we derive a concave-convex procedure~(CCCP) for uLNML minimization.
The algorithm is justified with the convergence properties that result from the CCCP framework.
Then, we also give concrete examples of the computation needed in the CCCP for typical RERMs.

\subsection{Concave-convex Procedure (CCCP) for uLNML Minimization}

In the forthcoming discussion, we assume that $\domainI$ is closed, bounded, and convex for computational convenience.
We also assume that the upper bound of the normalizing factor $\log \Zbar(\lambda)$ is convex with respect to $\lambda$.
This is not a restrictive assumption as the true normalizing term $\log Z(\lambda)=\log \int \exp\cbr{\max_{\theta\in\domainP} -f_X(\theta) - g(\theta,\lambda)}d\nu(X)$ is always convex if the penalty is linear as given in \eqref{eq:linear_penalty}. In particular, it is actually convex for the Tikhonov regularization and lasso as in Corollary~\ref{cor:l2} and Corollary~\ref{cor:lasso}, respectively.

Recall that the objective function, uLNML, is written as
\begin{align*}
    \uLNML(X|\lambda) = \min_{\theta\in\domainP} f_X(\theta)+g(\theta,\lambda)+\log \Zbar(\lambda).
\end{align*}
Therefore, the goal is to find $\lambda^\star\in\domainI$ that attains
\begin{align*}
    \uLNML(X|\lambda^\star) = \min_{\theta\in\domainP, \lambda\in \domainI} f_X(\theta)+g(\theta,\lambda)+\log \Zbar(\lambda),
\end{align*}
as well as the associated RERM estimate $\theta^\star=\thetahat(X, \lambda^\star)$.
Note that the existence of $\lambda^\star$ follows from the continuity of the objective function $\uLNML(X|\lambda)$ and the closed nature of the domain $\domainI$.

The minimization problem can be solved by alternate minimization of $h_X$ with respect to $\theta$ and $\lambda$ as in Algorithm~\ref{alg:uLNML_minimization}, which we call MDL regularization selection~(MDL-RS).
In general, minimization with respect to $\theta$ is the original RERM~\eqref{eq:RERM} itself.
Thus, it can often be solved with existing software or libraries associated with the RERM problem.
On the other hand, for minimization with respect to $\lambda$,
we can employ standard convex optimization techniques since $h_X(\theta, \cdot)$ is convex as both $g(\theta, \cdot)$ and $\log \Zbar(\cdot)$ are convex.
Specifically, for some types of penalty functions, we can derive closed-form formulae.
If one employs the Tikhonov regularization and $\Hupper_0$ is diagonal,
then
\begin{align*}
    \pwrt{\lambda_j} \sbr{g(\theta_t, \lambda) + \Zbar(\lambda)}=0
    \Leftrightarrow \lambda = \frac{\Hupper_{0, jj}}2 \sbr{\sqrt{1 + \frac{4}{\theta_{t,j}^2 \Hupper_{0, jj}}} - 1} \quad\rbr{= \lambdatil_{t,j}}.
\end{align*}
Therefore, if $\domainI=[a_1, b_1]\times \cdots \times [a_d, b_d]$,
the convex part is completed by $\lambda_{t,j}=\Pi_{[a_j, b_j]}\lambdatil_{t,j}$, where $\Pi_{[a_j, b_j]}$ is the projection of the $j$-th coordinate.
Similarly, we also have a formula for the lasso, 
\begin{align*}
    \lambdatil_{t, j} = \sqrt[3]{\alpha + \sqrt{\alpha^2 + \rbr{\frac{\Hupper_{0, jj}}{3}}^3}}
    + \sqrt[3]{\alpha - \sqrt{\alpha^2 + \rbr{\frac{\Hupper_{0, jj}}{3}}^3}},
\end{align*}
where $\alpha=\Hupper_{0, jj}/\abs{\theta_{t,j}}$.
The projection procedure is the same as that for Tikhonov regularization. 

\begin{algorithm}[htbp]                      
    \caption{MDL regularization selection (MDL-RS)}          
\label{alg:uLNML_minimization}                           
\begin{algorithmic}[1]                    
    \REQUIRE $X\in\domainX,\;\lambda_0\in \domainI$
    \STATE $t\leftarrow0$
    \MYREPEATUNTIL{
        \STATE $t\leftarrow t + 1$
        \STATE $\theta_t\leftarrow \argmin_{\theta\in\domainP} f_X(\theta) + g(\theta, \lambda_{t-1})$
        \STATE $\lambda_t\leftarrow \argmin_{\lambda\in\domainI} g(\theta_{t}, \lambda) + \log \Zbar(\lambda)$
    } {stopping condition is met}
    \RETURN $\theta_t, \lambda_t$
\end{algorithmic}
\end{algorithm}

The MDL-RS algorithm can be regarded as a special case of concave-convex procedure~(CCCP)~\cite{yuille2003concave}.
First, the RERM objective is concave as it is the minimum of linear functions,
$F_X(\lambda)=\min_{\theta\in \domainP} f_X(\theta) + g(\theta, \lambda)$.
Hence, uLNML is decomposed into the sum of concave and convex functions,
\begin{align*}
    \uLNML(X|\lambda)=F_X(\lambda)+\log \Zbar(\lambda).
\end{align*}
Second, $\Ftil^{(t)}_X(\lambda)\eqdef f_X(\theta_t) + g(\theta_t, \lambda)$ is a linear majorization function of $F_X(\lambda)$ at $\lambda=\lambda_{t-1}$,
i.e., $\Ftil^{(t)}_X(\lambda)\ge F_X(\lambda)$ for all $\lambda\in\domainI$ and $\Ftil^{(t)}_X(\lambda_{t-1})=F_X(\lambda_{t-1})$.
Therefore, as we can write $\lambda_{t}=\argmin_{\lambda\in \domainI}\Ftil^{(t)}_X(\lambda) + \log \Zbar(\lambda)$, MDL-RS is a concave-convex procedure for minimizing uLNML.

The CCCP interpretation of MDL-RS immediately implies the following convergence arguments.
Please refer to \cite{yuille2003concave} for the proofs.

\begin{theorem}[Monotonicity of MDL-RS]
    Let $\cbr{\lambda_t}_{t=0}^{\infty}$ be the sequence of solutions produced by Algorithm~\ref{alg:uLNML_minimization}. Then, we have $\uLNML(X|\lambda_{t+1}) \le \uLNML(X|\lambda_{t})$ for all $t\ge 0$.
\end{theorem}
\begin{theorem}[Local convergence of MDL-RS]
    Algorithm~\ref{alg:uLNML_minimization} converges to one of the stationary points of uLNML $\uLNML(X|\lambda)$.
\end{theorem}

Even if the concave part, i.e., minimization with respect to $\theta$, cannot be solved exactly,
MDL-RS still monotonically decreases uLNML as long as the concave part monotonically decreases the objective value,
$f_X(\theta_{t})+g(\theta_{t}, \lambda_{t-1}) \le f_X(\theta_{t-1})+g(\theta_{t-1}, \lambda_{t-1})$ for all $t\ge 1$.
This can be confirmed by seeing that $\uLNML(X|\lambda_{t})=h_X(\theta_{t+1}, \lambda_{t})\le h_X(\theta_t, \lambda_t)\le h_X(\theta_t, \lambda_{t-1})=\uLNML(X|\lambda_{t-1})$.
On the contrary, if the subroutine of the concave part is iterative, early stopping may beneficial in terms of the computational cost.

\subsection{Discussion}

We previously introduced the CCCP algorithm for minimizing uLNML, namely, MDL-RS.
The monotonicity and local convergence property follow from the CCCP framework.
One of the most prominent features of the MDL-RS algorithm is that the concave part is left completely black-boxed.
Thus, it can be easily applied to the existing RERM.

There exists another approach for minimization of LNMLs in which a stochastic minimization algorithm is proposed~\cite{miyaguchi2017sparse}.
Instead of approximating the value of LNML, this directly approximates the gradient of LNML with respect to $\lambda$ in a stochastic manner.
However, since the algorithm relies on the stochastic gradient,
there is no trivial way of judging if it is converged or not.
On the other hand, MDL-RS can exploit the information of the exact gradient of uLNML to stop the iteration.

Approximating LNML with uLNML benefits us more;
We can combine MDL-RS with grid search.
Since MDL-RS could be trapped at fake minima, i.e., local minima and saddle points,
starting from multiple initial points may be helpful to avoid poor fake minima,
and help it achieve lower uLNML.

\section{Related Work}\label{sec:related}

As compared to existing methods, MDL-RS is distinguished by its efficiency in searching for penalties and its ease of systematic computation.
Conventional penalty selection methods for large-dimensional models are roughly classified into three categories.
Below, we briefly describe each one emphasizing its relationship and differences with the MDL-RS algorithm.

\subsection{Grid Search with Discrete Model Selection Criteria}

The first category is grid search with a discrete model selection criterion
such as the cross validation score, Akaike's information criterion~(AIC)~\cite{akaike1974new}, and Bayesian information criterion~(BIC)~\cite{schwarz1978estimating}, \cite{chen2008extended}.
In this method, we choose a model selection criterion and a candidate set of the hyperparameter $\{\lambda^{(k)}\}_{k=1}^K\in\domainI\subset \RR^d$ in advance.
Then, we calculate the RERM estimates for each candidate, $\theta^{(k)}=\thetahat(X, \lambda^{(k)})$.
Finally, we pick the best estimate according to the pre-determined criterion.
This method is simple and universally applicable for any model selection criteria.
However, the time complexity grows linearly as the number of candidates increases,
and an appropriate configuration of the candidate set can vary corresponding to the data.
This is specifically problematic for high dimensional design spaces, i.e., $d\gg 1$,
where the combinatorial number of possible configurations is much larger than the feasible number of candidates.

On the other hand, the computational complexity of MDL-RS often scales better.
Though it depends on the time complexity of the subroutine for the original RERM problem, the MDL-RS algorithm is not explicitly affected by the curse of dimensionality.
However, it can be used for model selection in combination with the grid search.
Although MDL-RS provides a more efficient way to seek a good $\lambda$ in a (possibly) high-dimensional space as compared to simple grid search, it is useful to combine the two.
Since uLNML is nonconvex in general, MDL-RS may converge to a fake minimum such as local minima and saddle points
depending on the initial point $\lambda_0$.
In this case, starting MDL-RS with multiple initial points $\lambda_0=\lambda^{(k)}$ may improve the objective value.

\subsection{Evidence Maximization}

The second category is evidence maximization.
In this methodology, one interprets the RERM as a Bayesian learning problem.
The approach involves converting loss functions and penalty functions into conditional probability density functions
$p(X|\theta)=e^{-f_X(\theta)}$ and prior density functions $p(\theta;\lambda)=e^{-g(\theta, \lambda)}(\int e^{-g(\psi, \lambda)}d\psi)^{-1}$, respectively.
Then, the evidence is defined as $p(X;\lambda)=\int p(X|\theta)p(\theta;\lambda)d\theta$ and it is maximized with respect to $\lambda$.
A successful example of evidence maximization is the relevance vector machine~(RVM) proposed by \cite{tipping2001sparse}. It is a Bayesian interpretation of the ridge regression with different penalty weights $\lambda_j$ on different coefficients, as described in Corollary~\ref{cor:l2}. This results in so-called automatic relevance determination, and makes the approach applicable to redundant models.

The maximization of the evidence can also be thought of as an instance of the MDL principle,
as it is equivalent to minimizing $-\log p(X;\lambda)$ with respect to $\lambda$, which is a code-length function of $X$.
Moreover, both LNML and the evidence have intractable integral in it.
A notable difference between the two is the computational cost to optimize them.
Though LNML contains an intractable integral in its normalizing term $\log Z(\lambda)$,
it can be systematically approximated by uLNML and uLNML is efficiently minimized via CCCP.
On the other hand, in the case of evidence, we do not know of any approximation that is as easy to optimize and as systematic as uLNML.
Even though a number of approximations have been developed for evidence
such as the Laplace approximation, variational Bayesian inference~(VB), and Markov chain Monte Carlo sampling~(MCMC),
these tend to be combined with grid search~(e.g., \cite{yuan2005efficient})
except for some special cases such as the RVM and Gaussian processes~\cite{rasmussen2006gaussian}.

\subsection{Error Bound Minimization}

The last category is error bound minimization.
The generalization capability of RERM has been extensively studied in bounding generalization errors specifically on the basis of the PAC learning theory~\cite{valiant1984theory} and PAC-Bayes theory~\cite{shawe1997pac}, \cite{mcallester1999pac}.
There also exist a considerable number of studies that relate error bounds with the MDL principle, including (but not limited to) \cite{barron1991minimum}, \cite{yamanishi1992learning} and \cite{chatterjee2014information}.
One might determine the hyperparamter $\lambda$ by minimizing the error bound.
However, being used to prove the learnability of new models,
such error bounds are often not used in practice more than the cross validation score.
MDL-RS can be regarded as an instance of the minimization of an error bound.
Actually, uLNML bounds the LNML code length, which was recently shown to be bounding the generalization error of the RERM estimate under some conditions including boundedness of the loss function and fidelity of hypothesis classes~\cite{grunwald2017tight}.

\section{Experiments}

In this section, we empirically investigate the performance of the MDL-RS algorithm in selecting penalty functions.\footnote{The source codes and datasets of the following experiments are available at \texttt{https://github.com/koheimiya/pymdlrs}.}
In particular, we are interested in how well MDL-RS performs in terms of generalization errors of RERM estimates.
To verify this, we compare MDL-RS with conventional methods applying the two models introduced in Section~\ref{sec:example} on both synthetic and benchmark datasets.

\subsection{Linear Regression}

For the linear regression, we compared MDL-RS with the automatic relevance determination~(ARD) regression with the relevance vector machine~(RVM)~\cite{tipping2001sparse} and grid search for the cross validation score and Bayesian information criterion. As for the cross validation, we employed the ridge regression and lasso while their penalty weights are configured as 20 points spread over $\lambda_j=\lambda\in[10^{-4}, 10^0]\ (j=1,\ldots,p)$ logarithmically evenly. The performance metric is test root-mean-squared-error~(RMSE), where 10\% of the total sample is used for the test with 10-fold cross validation.
Figure~\ref{fig:regression} shows the results of the comparison with four datasets, namely, two synthetic datasets with uncorrelated and correlated features, the Diabetes dataset\footnote{http://www4.stat.ncsu.edu/~boos/var.select/diabetes.html} and the Boston Housing dataset.\footnote{https://www.cs.toronto.edu/~delve/data/boston/bostonDetail.html}
In the synthetic datasets, there are five informative features and $45$ non-informative, completely irrelevant ones.
Specifically, in the correlated case, all the features are distributed in a 10-dimensional subspace while the remaining setting is the same as in the uncorrelated case.

From the overall results, we can see that MDL-RS and the sparse Bayesian regression are comparable to one another and outperform the other two.
Figure~\ref{fig:regression_synth_well_conditioned} suggests that the proposed method performs well in both synthetic experiments.
In particular, one can see that MDL-RS and RVM converge faster in Figure~\ref{fig:regression_synth_singular} than in Figure~\ref{fig:regression_synth_well_conditioned}.
This is corresponding to the fact that regression with (completely) correlated features is more redundant than that with uncorrelated features because the effective number of feature is degenerated.
Figure~\ref{fig:regression_real_diabetes} and Figure~\ref{fig:regression_real_boston} show the results of the benchmark experiments. From these experiments, one can observe the same tendency as from the synthetic ones; Both MDL-RS and RVM outperforms the rest and the difference is bigger when the sample size is smaller.

\begin{figure}[htbp]
    \centering
    \subfloat[Synthetic data with uncorrelated features]{{
            \includegraphics[width=60mm]{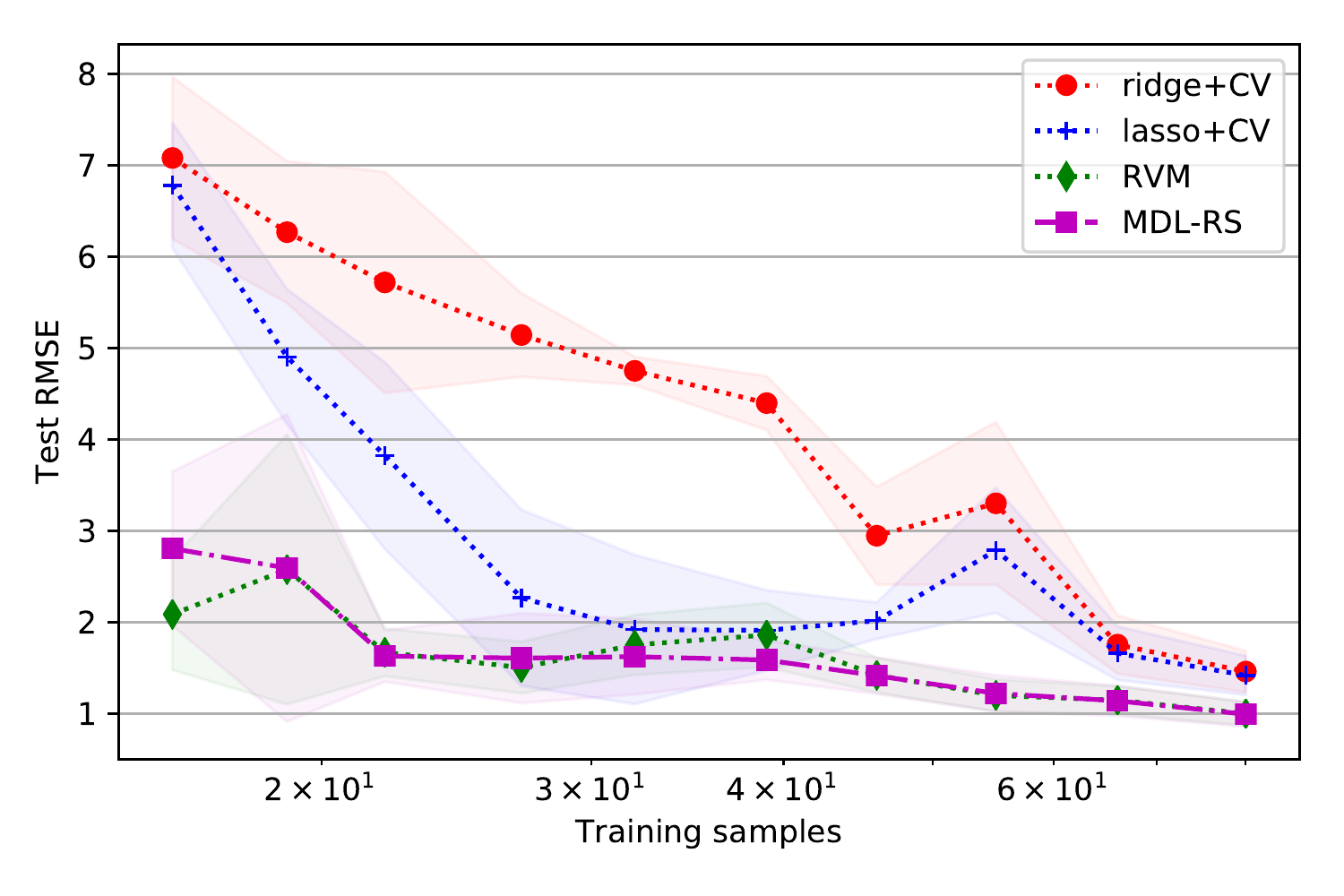}
            \label{fig:regression_synth_well_conditioned}
    }}
    \subfloat[Synthetic data with correlated features]{{
            \includegraphics[width=60mm]{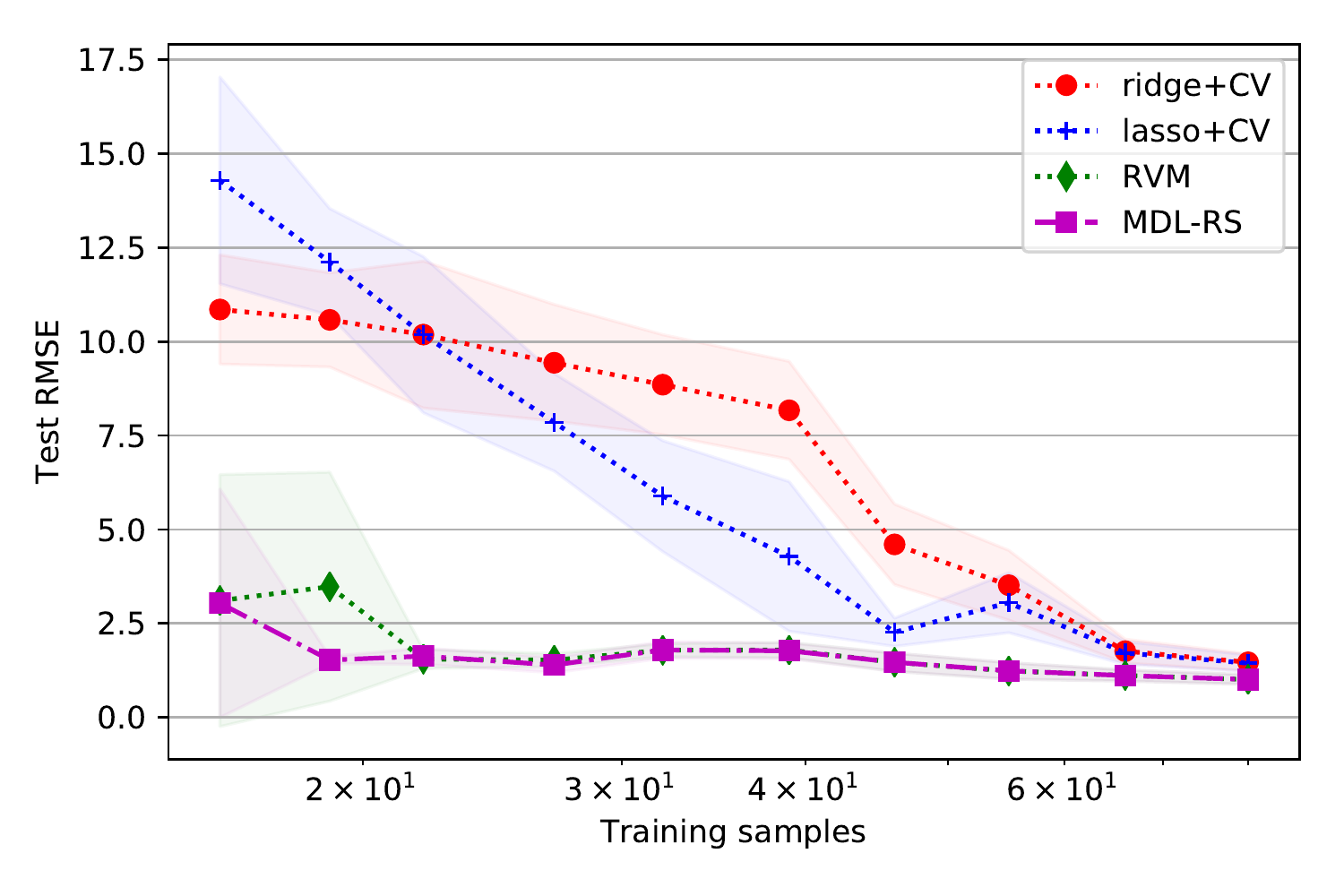}
            \label{fig:regression_synth_singular}
    }}\\
    \subfloat[Diabetes data]{{
            \includegraphics[width=60mm]{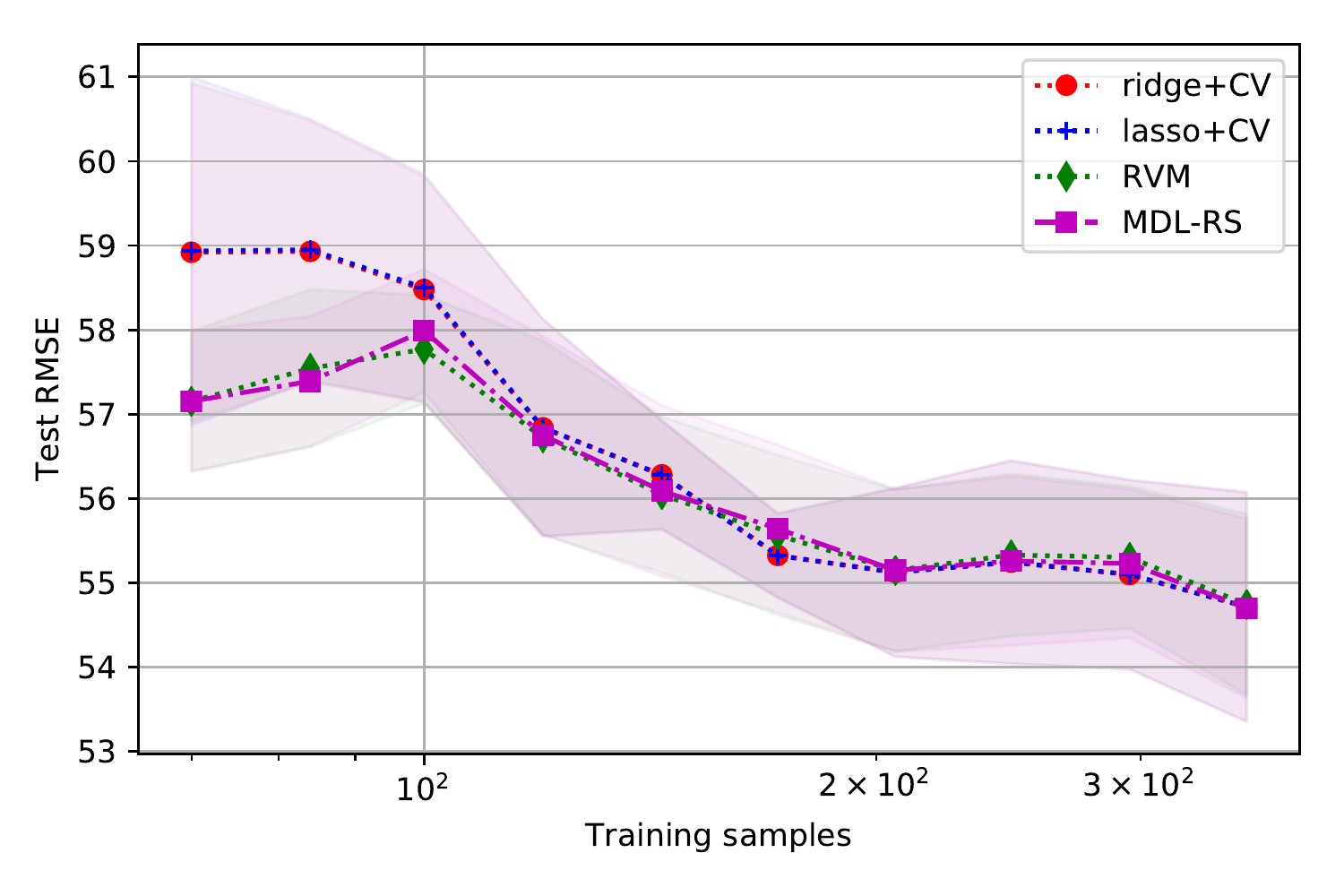}
            \label{fig:regression_real_diabetes}
    }}
    \subfloat[Boston data]{{
            \includegraphics[width=60mm]{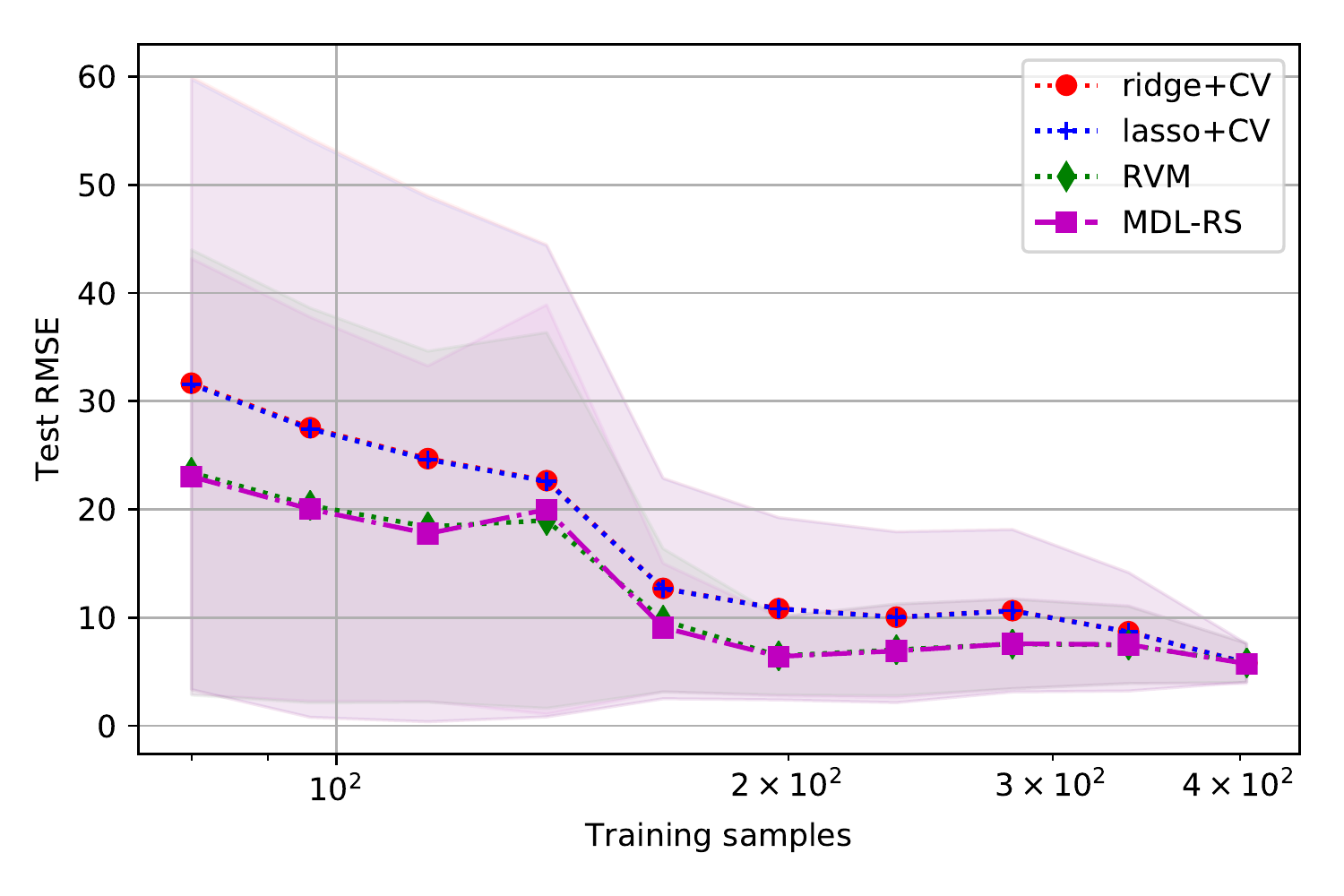}
            \label{fig:regression_real_boston}
    }}
    \\
    \caption{
        \textbf{Convergence of RMSE in linear regression}\\
        The horizontal axes show the number of training samples in logarithmic scale,
        while the vertical axes show the test RMSE.
        Each shading area is showing $\pm$one-standard deviation.
    }
    \label{fig:regression}       
\end{figure}

\subsection{Conditional Dependence Estimation}

For the estimation of conditional dependencies,
we compared MDL-RS with the grid search of glasso~\cite{friedman2008sparse} with AIC, (extended) BIC and the cross validation score.
We generated data $X\in\RR^{n\times m}$ from $m$-dimensional double-ring Gaussian graphical models~($m=10, 20, 50, 100$)
in which each variable $j\in[1, m]$ is conditionally dependent to its 2-neighbors $j-2, j-1, j+1$ and $j+2\;(\mathrm{mod}\ m)$ with a coefficient of $0.25$.
Note that MDL-RS can be applied to the graphical model just by computing the upper smoothness while RVM cannot be applied straightforwardly.

Figure~\ref{fig:graphical_synthetic} shows the results of the experiment.
It is seen that all the estimators converge in the same rate, $O(n^{-1})$,
whereas MDL-RS gives the least Kullback--Leibler divergence by far specifically with large $m$.
In particular, when $m=100$, the proposed estimator outperforms the other by more than a factor of five.
This supports our claim that penalty selection in high-dimensional design spaces has a considerable effect on the generalization capability when the model is redundant.


\begin{figure}[htbp]
    \centering
    \subfloat[$m=10$]{{ \includegraphics[width=60mm]{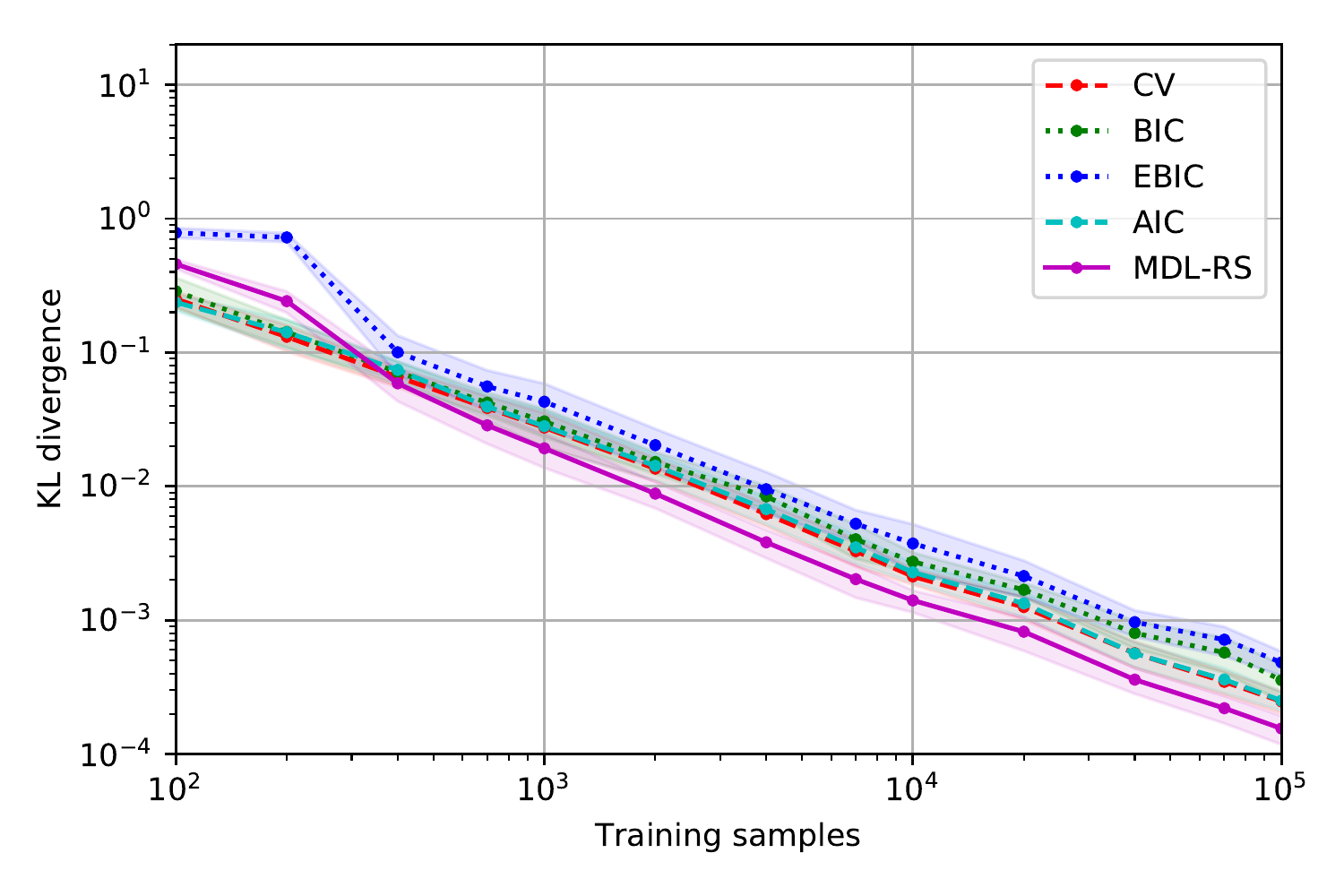} }}
    \subfloat[$m=20$]{{ \includegraphics[width=60mm]{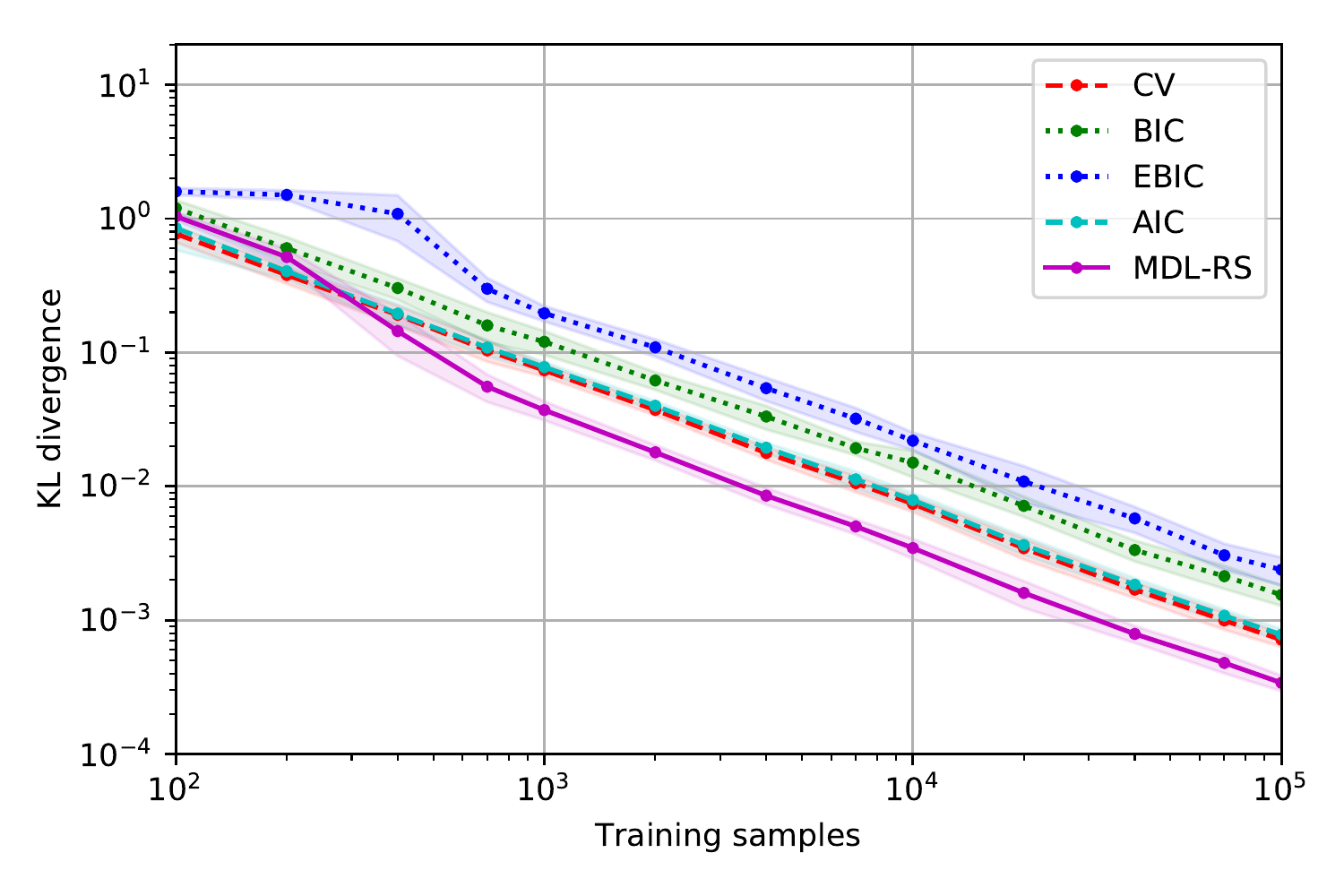} }}
    \\
    \subfloat[$m=50$]{{ \includegraphics[width=60mm]{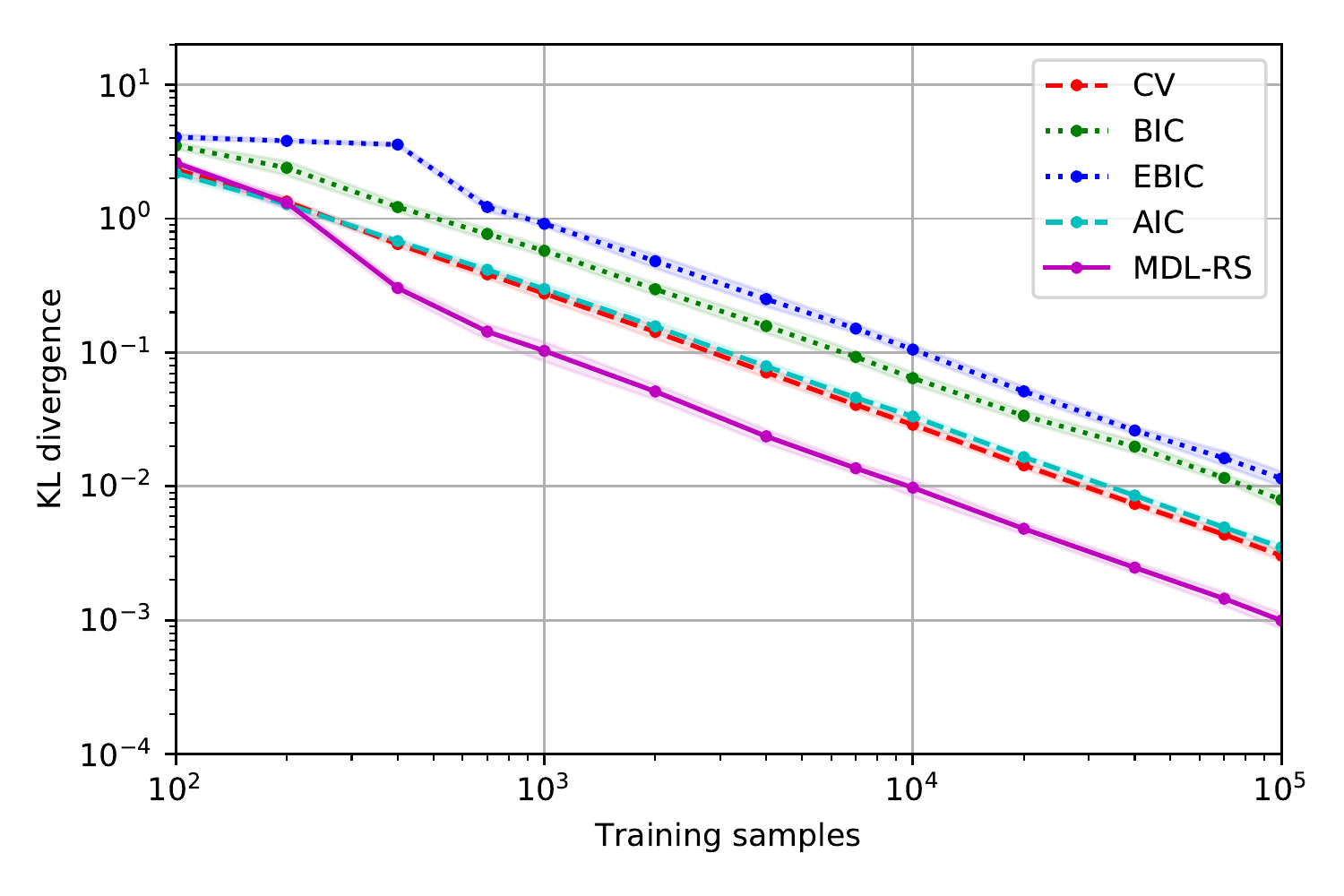} }}
    \subfloat[$m=100$]{{ \includegraphics[width=60mm]{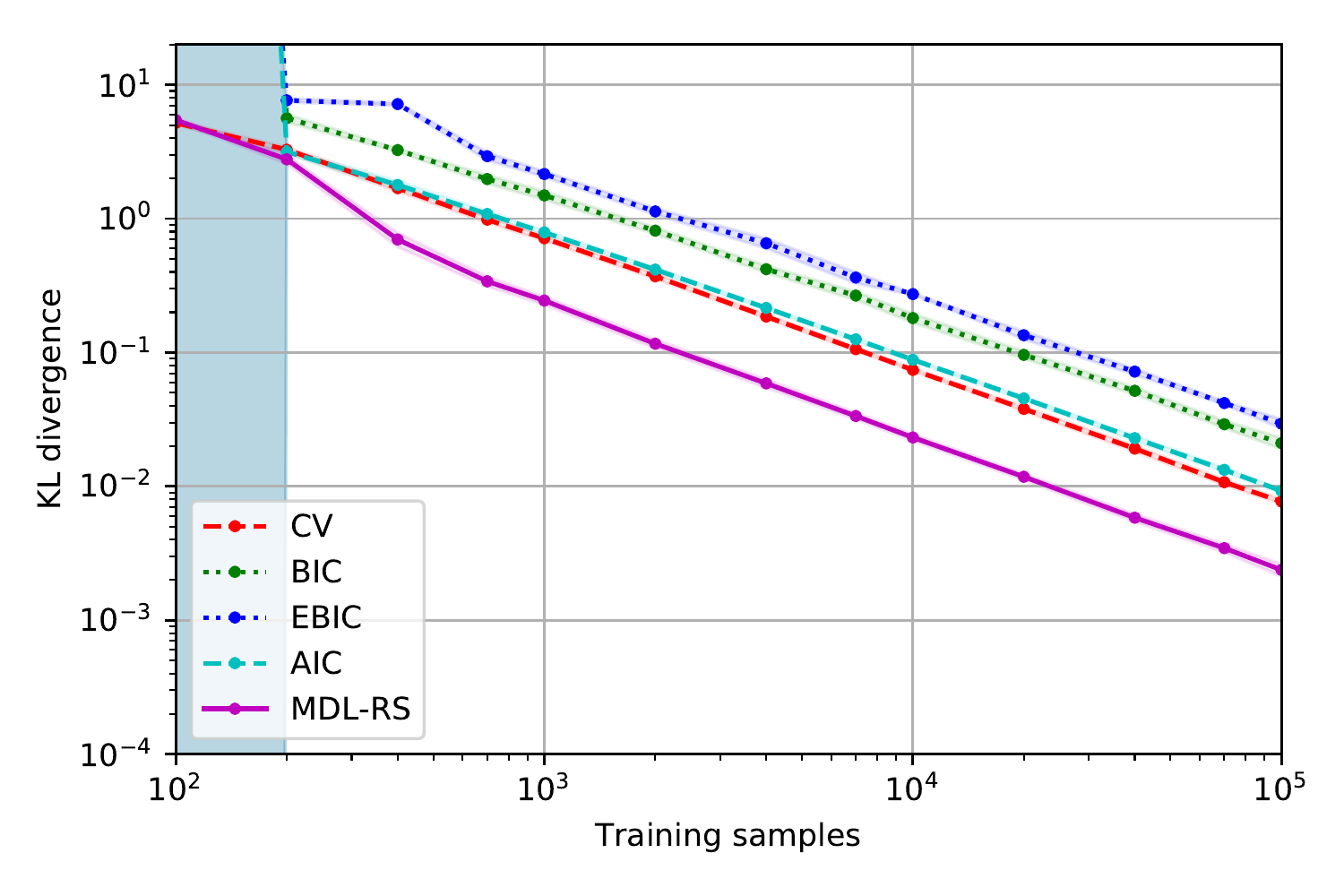} }}
    \\
    \caption{
        \textbf{Convergence of Kullback--Leibler divergence for graphical models}\\
        The horizontal axes show the number of training samples in logarithmic scale,
        while the vertical axes show the divergence of estimates relative to true distributions.
        Each shading area is showing $\pm$one-standard deviation.
    }
    \label{fig:graphical_synthetic}       
\end{figure}

\section{Concluding Remarks}

In this paper, we proposed a new method for penalty selection on the basis of the MDL principle.
Our main contribution was the introduction of uLNML, a tight upper bound of LNML for smooth RERM problems.
This can be analytically computed, except for a constant, given the (upper) smoothness of the loss and penalty functions.
We also presented the MDL-RS algorithm, a minimization algorithm of uLNML with convergence guarantees.
Experimental results indicated that MDL-RS's generalization capability was comparable to that of conventional methods. In the high-dimensional setting we are interested in, it even outperformed conventional methods.

In related future work, further applications to various models such as latent variable models and deep learning models must be analyzed. As the above models are not (strongly) convex, the extension of the lower bound of LNML to the non-convex case would also be an interesting topic of future study.
While we bounded LNML with the language of Euclidean spaces, the only essential requirement of our analysis is upper smoothness of loss functions defined over parameter spaces.
Therefore, we believe that it is possible to generalize uLNML to the Hilbert spaces to deal with infinite-dimensional models.


\bibliographystyle{spbasic}      
\bibliography{reference}   

\end{document}